\documentclass{article} 

\usepackage[dvipsnames]{xcolor}

\usepackage{iclr2024_workshop,times}

\def\viewchanges{1}
\def\viewauthors{1}

\def\usehyperliks{1}
\def\preprint{1}

\usepackage{amsmath,amsfonts,bm}

\def\1{\bm{1}}

\DeclareMathAlphabet{\mathsfit}{\encodingdefault}{\sfdefault}{m}{sl}
\SetMathAlphabet{\mathsfit}{bold}{\encodingdefault}{\sfdefault}{bx}{n}

\newcommand{\E}{\mathbb{E}}

\newcommand{\R}{\mathbb{R}}

\usepackage{url}            
\usepackage{booktabs}       
\usepackage{amsfonts}       
\usepackage{nicefrac}       
\usepackage{microtype}      
\usepackage{amsthm}
\usepackage{amsmath}
\usepackage{bbm}
\usepackage{enumitem}

\if\usehyperliks1
	\usepackage[colorlinks, citecolor=blue, linkcolor=magenta]{hyperref}       
\fi
\definecolor{william}{rgb}{1.,0.,0.}
\definecolor{florian}{rgb}{0.,1.,0.}

\usepackage{graphicx,wrapfig,lipsum}
\usepackage{subfigure}
\usepackage{verbatim}
\usepackage{bm}
\usepackage{adjustbox}
\usepackage{footnote}
\usepackage{algorithm}
\usepackage{algorithmic}
\usepackage{cleveref}

\newtheorem{theorem}{Theorem}[section]

\newtheorem{prop}[theorem]{Proposition}
\newtheorem{rem}[theorem]{Remark}
\newtheorem{cor}[theorem]{Corollary}

\let\P\undefined%
\renewcommand{\1}{\mathbf{1}}
\newcommand{\P}{\mathbb{P}}

\renewcommand{\E}{\mathbb{E}} 
\newcommand{\N}{\mathbb{N}} 
\renewcommand{\R}{\mathbb{R}}

\let\del\undefined
\let\com\undefined

\if\viewchanges1
	
	\newcommand{\del}[1]{{\color{red}{#1}}}
	\newcommand{\com}[1]{{\color{orange}{#1}}}
	
\else
	
	\newcommand{\del}[1]{}
	\newcommand{\com}[1]{}
	
\fi

\title{Learning Chaotic Systems and Long-Term Predictions with Neural Jump ODEs}

\author{Florian Krach \\
Department of Mathematics \\
ETH Zurich \\
Zurich, Switzerland\\
\texttt{florian.krach@math.ethz.ch}
\And 
Josef Teichmann \\
Department of Mathematics \\
ETH Zurich \\
Zurich, Switzerland\\
\texttt{josef.teichmann@math.ethz.ch}
}

\if\viewauthors1
    \if\preprint1
        \iclrpreprint
    \else
        \iclrfinalcopy 
    \fi
\fi

\begin{document}

\maketitle

\begin{abstract}
The Path-dependent Neural Jump ODE (PD-NJ-ODE) is a model for online prediction of generic (possibly non-Markovian) stochastic processes with irregular (in time) and potentially incomplete (with respect to coordinates) observations. It is a model for which convergence to the $L^2$-optimal predictor, which is given by the conditional expectation, is established theoretically. Thereby, the training of the model is solely based on a dataset of realizations of the underlying stochastic process, without the need of knowledge of the law of the process. In the case where the underlying process is deterministic, the conditional expectation coincides with the process itself. Therefore, this framework can equivalently be used to learn the dynamics of ODE or PDE systems solely from realizations of the dynamical system with different initial conditions. We showcase the potential of our method by applying it to the chaotic system of a double pendulum. When training the standard PD-NJ-ODE method, we see that the prediction starts to diverge from the true path after about half of the evaluation time. In this work we enhance the model with two novel ideas, which independently of each other improve the performance of our modelling setup. The resulting dynamics match the true dynamics of the chaotic system very closely.
The same enhancements can be used to provably enable the PD-NJ-ODE to learn long-term predictions for general stochastic datasets, where the standard model fails. This is verified in several experiments.
\end{abstract}

\section{Introduction}\label{sec:Introduction}
The Path-dependent Neural Jump ODE (PD-NJ-ODE) \citep{krach2022optimal} is a model for online prediction of generic (possibly non-Markovian) stochastic processes with irregular and potentially incomplete observations. It is the first model for which convergence to the $L^2$-optimal predictor, which is given by the conditional expectation, is established theoretically. Thereby, the training of the model is solely based on a dataset of realizations of the underlying stochastic process, without the need of knowledge of the law of the process. This result was further generalized in the follow-up work \citet{andersson2024extending}.
In particular, let $(X_t)_{t \in [0,T]}$ be a stochastic process taking values in $\R^d$, let $t_i \in [0,T]$ be random observations times for $1 \leq i \leq n$, where $n$ can be a random variable describing the total number of observations (i.e., the number of observations can be different for each realization) and let $M_i \in \{0,1\}^d$ be the corresponding random observation masks, telling which coordinates $X_{t_i,j}$ are observed (if $M_{i_j} =1$) at each observation time $t_i$. 
The $\sigma$-algebra of the currently available information at any time $t \in [0,T]$ is defined as 
\begin{equation*}
\mathcal{A}_{t} := \boldsymbol{\sigma}\left(X_{t_{i}, j}, t_{i}, M_{t_{i}} | t_{i} \leq t,\, j \in \{1 \leq l \leq d | M_{t_{i}, l} = 1  \} \right),
\end{equation*} 
where $\boldsymbol\sigma(\cdot)$ denotes the generated $\sigma$-algebra. 
Then, \citet[Theorem~4.3]{andersson2024extending} states that the output $Y^{\theta^{\min}_{m,N_m}}$ of the PD-NJ-ODE model (where $\theta$ are the trainable parameters of the model, $m$ is the size of the used neural networks and $N_m$ is the number of training paths, i.e., realizations of $X$) converges to the $L^2$-optimal prediction $\hat{X} = (\E[ X_t | \mathcal{A}_t ])_{t \in [0,T]}$ as $m$ tends to $\infty$. 
This convergence holds under weak assumptions on $X$ and the observation framework $(t_i, M_i, n)$, which basically require some integrability properties and continuous differentiability of $t \mapsto \hat{X}_t$.

In \citet{krach2022optimal, andersson2024extending}, the focus lies on optimal prediction of generic stochastic processes, as for example processes defined via an stochastic differential equation, given the currently available information.
In particular, this means that the model never predicts further than until the next observation time, since then the available information changes. If the next observation time is deterministically (or with very high probability) smaller than $r$, then it is unlikely that the model learns to predict well for $t > r$, without getting the new information as input, when it becomes available at the next observation time.
In this work, we focus on a provable training strategy, that makes such long-term predictions possible.

This is of particular importance in the case of a deterministic (given the initial condition) underlying process, as in (chaotic) ODE or PDE systems.
Importantly, in this setting the conditional expectation coincides with the process itself. 
In particular, if $X_0$ is observed, i.e., if $X_0$ is $\mathcal{A}_t$-measurable for any $t \in [0,T]$, then $\hat{X}_t = \E[ X_t | \mathcal{A}_t ] = X_t$. 
Therefore, \citet[Theorem~4.3]{andersson2024extending} implies that the PD-NJ-ODE framework can equivalently be used to learn the dynamics of ODE or PDE systems solely from realizations of the dynamical system with different initial conditions.
This result was already stated in \citet[Appendix~B.3]{krach2022optimal} and was used in the experiments in \citet[Appendix~C.3]{krach2022optimal}.
Even though the theoretical results are promising, it can be seen in the empirical results of \citet[Appendix~C.3]{krach2022optimal} (and in \Cref{fig:DP Comp} left) that the PD-NJ-ODE has problems to predict a chaotic system well over longer time periods, when the prediction is only based on the initial value. In particular, the prediction starts to diverge from the true path after about half of the evaluation time. 
The problem is that during the training, the model never needs to predict so far ahead (since it gets intermediate observations as input). Hence, it also does not learn to do this well.
In this work, we analyse the  PD-NJ-ODE model from the perspective of learning long-term predictions of stochastic or deterministic (differential) systems and introduce two novel ideas, which enhance the training of the model significantly and independently of each other in this context.

\subsection{Related work}\label{sec:Realted work}
This work is based on the sequence of papers introducing NJ-ODE \citep{herrera2021neural}, extending it to a path-dependent setting with incomplete observations \citep{krach2022optimal} and further to noisy observations with dependence between the observation framework and the underlying process \citep{andersson2024extending}.
The focus of this paper lies on long-term predictions (i.e., multiple observation times ahead), with a special emphasis on fully observed (chaotic) deterministic systems. 
The framework of \citep{krach2022optimal} also allows for partially observed chaotic systems, which are not deterministic. Such cases resemble stochastic processes, where the optimal prediction, given by the conditional expectation, is learnt. Hence, they can be treated with the provided result for general stochastic processes.

\citet{NAVONE1995383} were one of the first to use neural networks to learn chaotic dynamics and several other works followed using RNNs, reservoir computing and neural ODEs \citep{Vlachas2018Datadriven, Pathak2018Modelfree, VLACHAS2020191, brenner2022tractable, chen2018neural, Raissi2018Deep}.
\citet{Churchill2022} propose a memory-based residual deep neural network (DNN) architecture to learn chaotic systems from fully or partially observed data and apply it to the chaotic Lorenz 63 and 96 systems.
\citet{pmlr-v202-hess23a} use piece-wise linear RNNs together with teacher forcing to effectively learn chaotic dynamics and provide an extensive overview of related work and numerical comparison to many state-of-the-art models.
Our approach, using neural ODEs, is particularly related to \citet{chen2018neural} and \citet{ODERNN2019}. However, in contrast to all these methods, our approach comes with theoretical learning guarantees even in the most general case of irregularly and incompletely observed path-dependent stochastic processes.

\section{Main results}\label{sec:Main results}
In the results of \citet[Appendix~C.3]{krach2022optimal} we see that the empirical performance of the PD-NJ-ODE applied to chaotic systems could be improved, especially for long prediction horizons, even though the theoretical results suggest that the model should learn to predict chaotic systems correctly at any time. This is related to the inductive bias when training the model with a finite amount of training samples (see \citet[Appendix~B]{andersson2024extending} for more details on the inductive bias of the PD-NJ-ODE). In particular, even if the distribution of the observation times is such that it is (theoretically) possible to have very long periods without observations, the probability of experiencing this necessarily becomes smaller the larger the period is. Hence, the respective training samples where this happens are scarce and consequently the empirical results of the model fall short of the theoretical expectations.

Therefore, we suggest two enhancements of the PD-NJ-ODE model for learning long-term predictions in deterministic (differential) as well as stochastic systems. In \Cref{sec:Convergence of the PD-NJ-ODE in the special case of deterministic systems} we prove that in the deterministic case, the model only taking the initial value as input (and potentially some of the following ones), converges to the same limiting process as the standard model, since all the observations are still used in the loss function. This should improve the inductive bias of the training, since the model is now forced to predict further into the future.  
In \Cref{sec:Long-term predictions} we show that the same training enhancement also leads to accurate long-term predictions in stochastic systems.
Moreover, in \Cref{sec:Output feedback in the PD-NJ-ODE model} we discuss that using \emph{output feedback} (which is known to stabilize the training of dynamical systems) in the PD-NJ-ODE model framework, still yields the same theoretical results.

\subsection{Long-term predictions with PD-NJ-ODE}\label{sec:Long-term predictions with PD-NJ-ODE}
In our context, \emph{long-term predictions} always refer to predictions within the time horizon $[0,T]$, where for any $0\leq s \leq t \leq T$, the information available up to time $s$ is used to predict the process at time $t$. This is a generalization of the standard framework, where we have $s=t$, i.e., where predictions are based on all available information up to the prediction time. 
Importantly, we make no claim for $t>T$. To extend our results for $t > T$, additional assumptions on the time-invariance of the underlying system would be necessary, which we do not require here.

We start by discussing the special case of deterministic (chaotic) systems in \Cref{sec:Convergence of the PD-NJ-ODE in the special case of deterministic systems}, then propose a training procedure in \Cref{sec:Suggested training procedure} based on those insights, and finally show that the same method also applies in the general case of stochastic systems in \Cref{sec:Long-term predictions}.

\subsubsection{The special case of deterministic systems}\label{sec:Convergence of the PD-NJ-ODE in the special case of deterministic systems}
As described in \Cref{sec:Introduction}, the PD-NJ-ODE is a model that can be used to predict a stochastic process $X$ given its previous discrete and possibly incomplete observations summarized in $\mathcal{A}_t$ for any $t\in[0,T]$. In particular this model directly uses every new observation as input when the observation becomes available and predicts for all times afterwards based additionally on this new observation. In the setting of stochastic processes this behaviour makes perfect sense, since every new piece of information changes (improves) the following forecasts. However, in the setting of deterministic (differential) systems, which are fully determined by their initial value, using new observations as input for the PD-NJ-ODE model is (in principle) not needed, since they do not provide any new information about $X$. In particular, we have 
\begin{equation}\label{equ:cond prob identity}
    \hat{X}_t = \E[ X_t | \mathcal{A}_t ] = X_t = \E[ X_t | \mathcal{A}_0 ]
\end{equation}
if $\boldsymbol{\sigma}(X_0) \subseteq \mathcal{A}_0$.
This allows us to formulate the following corollary of \citet[Theorem~4.3]{andersson2024extending}.
\begin{cor}\label{cor:convergece for deterministic system}
    Under the same assumptions as in \citet[Theorem~4.3]{andersson2024extending} with the additional assumption that $X$ is deterministic given its initial value $X_0$, we denote by $\tilde{Y}^{\theta^{\min}_{m,N_m}}$ the output of the PD-NJ-ODE model, where only the fully observed initial value $X_0$ is used as input to the model (in the training). Then, the same convergence result holds for $\tilde{Y}^{\theta^{\min}_{m,N_m}}$ as for ${Y}^{\theta^{\min}_{m,N_m}}$ in \citet[Theorem~4.3]{andersson2024extending}. In particular, $\tilde{Y}^{\theta^{\min}_{m,N_m}}$ converges to $\hat{X}$ as $m \to \infty$.
\end{cor}
\begin{rem}
    We emphasize that all available observations of $X$ are still used in the loss function to train the model, they are only not used as input to the model. Therefore, we still have convergence in the metrics $d_k$ for all $1 \leq k \leq K$.
\end{rem}

\begin{proof}[Proof of \Cref{cor:convergece for deterministic system}.]
    First note that $X_0$ being fully observed implies that $\boldsymbol{\sigma}(X_0) = \mathcal{A}_0$.
    Revisiting the proof of \citet[Theorem~4.3]{andersson2024extending}, it is easy to see that the $L^2$-optimal $\boldsymbol{\sigma}(X_0)$-measurable prediction $(\E[ X_t | \mathcal{A}_0 ])_{t \in [0,T]}$ of $X$ is the unique minimizer (up to indistinguishability) of the loss function amongst all $\boldsymbol{\sigma}(X_0)$-measurable processes. Moreover, it follows as before that the PD-NJ-ODE model can approximate $(\E[ X_t | \mathcal{A}_0 ])_{t \in [0,T]}$ arbitrarily well. Therefore, training the PD-NJ-ODE model, which only takes $X_0$ as input, with the same training framework yields convergence of $\tilde{Y}^{\theta^{\min}_{m,N_m}}$ to $(\E[ X_t | \mathcal{A}_0 ])_{t \in [0,T]}$. Finally, \Cref{equ:cond prob identity} implies that $\tilde{Y}^{\theta^{\min}_{m,N_m}}$ actually converges to $\hat{X} = (\E[ X_t | \mathcal{A}_t ])_{t \in [0,T]}$.
\end{proof}

Clearly, for any PD-NJ-ODE model, taking $X_0$ and some of the following observations as input, the same convergence result holds, since the result holds for the two extreme cases of models taking all or none of the following observations as input.

While \Cref{cor:convergece for deterministic system} might seem to be a trivial extension of the original result, its practical importance is large in the context of learning deterministic (differential) systems. As outlined in the beginning of \Cref{sec:Main results}, the model which only takes $X_0$ as input is forced to learn to predict well over the entire time period. Hence, we effectively improve the inductive bias of the model without changing the theoretical guarantees.

\subsubsection{Suggested training procedure}\label{sec:Suggested training procedure}
We note that using the observations as input is not only disadvantageous but also has  a positive effect on the inductive bias. In particular, every (full) observation that is used as input for the model basically amounts to using this observation as new initial value for the system, hence, increasing the amount of initial values used to train the PD-NJ-ODE model.
This is particularly useful in the beginning of the training. Therefore, we introduce a probability $p \in [0,1]$ and use i.i.d.\ Bernoulli random variables $I_k \sim \text{Ber}(p)$, which decide whether an observation is used as input to the model during training. By decreasing the probability $p$ throughout the training we can therefore first use the observations as additional initial values and then focus the training more and more on predicting well over a long time period. 
Since there exists one solution which is optimal for all $p\in [0,1]$, this procedure additionally encourages the model to learn it.
The effectiveness of this procedure can be seen in \Cref{sec:Experiments}.
Nevertheless, we note that theoretically, choosing any fixed $p \in (0,1)$ leads to the same optimal solution, as proven in the following section.

\subsubsection{General stochastic systems}\label{sec:Long-term predictions}
Similarly as in the case of a deterministic (chaotic) $X$, also in the stochastic case, we might be interested in learning to predict multiple time steps ahead. In the standard framework, the PD-NJ-ODE model only learns to predict until the next observation time, since it converges to $\E[X_t | \mathcal{A}_t]$, which is the optimal prediction of $X_t$ given all information available up to $t$, i.e., all information gathered at observation times before or at $t$. However, the training procedure suggested in \Cref{sec:Suggested training procedure} allows to generalise this, such that the PD-NJ-ODE model learns to correctly predict 
\begin{equation}\label{equ:X-t-s}
    \hat{X}_{t,s} := \E[X_t | \mathcal{A}_{s \wedge t}],    
\end{equation}
for any $0 \leq s ,t \leq T$, which is shown in the following two results.

\begin{cor}\label{cor:convergece for multi times ahead prediction}
    Let $p \in (0,1)$ and $I_k \sim \operatorname{Ber}(p)$ be i.i.d.\ random variables for $k \in \N$, which are independent of $X$ and the observation framework $n, t_i, M_i$. 
    Under the same assumptions as in \citet[Theorem~4.3]{andersson2024extending}, with $\mathbb{A}$ replaced by $\tilde{\mathbb{A}}$ defined below, we denote by $\tilde{Y}^{\theta^{\min}_{m,N_m}}$ the output of the PD-NJ-ODE model, where $I_k$ determines whether the $k$-th observation is used as input to the PD-NJ-ODE model during training. In particular, the model only uses the information available in the $\sigma$-algebra
    \begin{equation*}
        \tilde{\mathcal{A}}_{t} := \boldsymbol{\sigma}\left(X_{t_{i}, j}, t_{i}, M_{t_{i}} | I_i = 1, t_{i} \leq t,\, j \in \{1 \leq l \leq d | M_{t_{i}, l} = 1  \} \right).
    \end{equation*} 
    We denote the corresponding filtration by $\tilde{\mathbb{A}}$.
    Then $\tilde{Y}^{\theta^{\min}_{m,N_m}}$ is $\tilde{\mathbb{A}}$-adapted and converges to the unique (up to indistinguishability) $\tilde{\mathbb{A}}$-adapted minimizer $t \mapsto \E[X_t | \tilde{\mathcal{A}}_t]$ of the loss function.
\end{cor}
\begin{proof}
    Adaptedness of the model output $\tilde{Y}^{\theta^{\min}_{m,N_m}}$ to $\tilde{\mathbb{A}}$ follows from the used input and the model architecture.
    The remainder of the statement follows equivalently as in the proof of \citet[Theorem~4.3]{andersson2024extending}. 
\end{proof}

In the following proposition we show that for all $s \in [0,T]$, $(\E[X_t | \tilde{\mathcal{A}}_t])_{0 \leq t \leq T}$ coincides with  $(\hat{X}_{t, s})_{0\leq t\leq T}$ conditioned on the event 
$B_s := \{ \forall k \leq n: I_k = \1_{\{ t_k \leq s \}}  \}$, which has positive probability. Hence, the PD-NJ-ODE model learns to predict $(\hat{X}_{t, s})_{0\leq t\leq T}$ on $B_s$.

\begin{prop}\label{prop:convergence multi times ahead}
    For all $s, t \in [0,T]$ we have $\P(B_s) > 0$ and $\P$-a.s.
    \begin{equation}\label{equ:equality for X-t-s}
    \begin{split}
        \1_{B_s} \E[X_t | \tilde{\mathcal{A}}_t] &= \1_{B_s} \hat{X}_{t, s}, \quad \text{and} \\
         \1_{B_s} \tilde{Y}^{\theta^{\min}_{m,N_m}}_t &=  \1_{B_s} \tilde{Y}^{\theta^{\min}_{m,N_m}}_t(\tilde X^{\leq t \wedge s}),
    \end{split}
    \end{equation}
    hence, $\1_{B_s} \left(\tilde{Y}^{\theta^{\min}_{m,N_m}}_t(\tilde X^{\leq t \wedge s})\right)_{0\leq t\leq T}$ converges \citep[as in][Theorem~4.3]{andersson2024extending} to $\1_{B_s} (\hat{X}_{t, s})_{0\leq t\leq T}$ as $m \to \infty$. 
\end{prop}

\begin{proof}
    Fix $s \in [0,T]$. $B_s$ can be written as the disjoint union $B_s = \cup_{m \geq 1} \cup_{0 \leq k \leq m} \{ n=m, \tau(s)=t_k, \forall j \leq m: I_j = \1_{\{ t_k \leq s \}}\}$,
    hence, independence of $I_k$ to $n$ and $t_i$ implies
    \begin{equation*}
        \P(B_s) = \sum_{m \geq 1} \sum_{k = 0}^m \P(n = m, \tau(s) = t_k) p^k (1-p)^{m-k} > 0, 
    \end{equation*}
    where $\tau(s)$ is the last observation time before or at time $s$, which shows the first part of the claim.
    Next, we note that on $B_s$ we have $\tilde{\mathcal{A}}_t = \mathcal{A}_{t \wedge s}$, since all observations before and no observations after $s$ are ``used''. By the same reasoning we have on $B_s$ that $\tilde{Y}^{\theta^{\min}_{m,N_m}}_t(\tilde X^{\leq t \wedge s}) = \tilde{Y}^{\theta^{\min}_{m,N_m}}_t(\tilde X^{\leq t}) = \tilde{Y}^{\theta^{\min}_{m,N_m}}_t$, hence \eqref{equ:equality for X-t-s} follows. Finally, we show the convergence in the metrics $d_k$ for any $1 \leq k \leq K$. With \eqref{equ:equality for X-t-s} we have
    \begin{equation*}
    \begin{split}
        d_k &\left( \1_{B_s} \tilde{Y}^{\theta^{\min}_{m,N_m}}_\cdot(\tilde X^{\leq \cdot \wedge s}), \1_{B_s} \hat{X}_{\cdot, s} \right) 
        = d_k \left( \1_{B_s} \tilde{Y}^{\theta^{\min}_{m,N_m}}, \1_{B_s} \E[X_t | \tilde{\mathcal{A}}_t] \right) \\
        & \qquad \leq d_k \left( \tilde{Y}^{\theta^{\min}_{m,N_m}},  \E[X_t | \tilde{\mathcal{A}}_t] \right) \xrightarrow{m \to \infty} 0,
    \end{split}
    \end{equation*}
    where the convergence follows from \Cref{cor:convergece for multi times ahead prediction}.
\end{proof}

\begin{rem}
    There are many equivalent options to choose the observations that are used as input to the model.
    Selecting them via i.i.d.\ Bernoulli random variables is one possibility that we use due to its simplicity. However, the same results can be derived with any other method of choosing the observations as inputs, as long as the probability of arbitrarily long periods without new inputs is positive, i.e., $\P(B_s) >0$ for all $s \in [0,T]$ (where the $I_k$ are defined through the chosen method).
    
    One explicit alternative method is to use exponentially distributed random variables to determine the time within which no observations are used as input. In particular, assuming that the current observation at $t_i$ is used as input and that $e_i \sim \operatorname{Exp}(\lambda)$ for some $\lambda > 0$, the next observation that is used as input is at the first observation time $t_k$ such that $t_k - t_i \geq e_i$. 
    This sub-sampling procedure has the advantage that the probability of not using any observation as input during a certain period only depends on the length of the period but not on the amount of observations during this period (as is the case for the i.i.d.\ Bernoulli random variables). 
\end{rem}

\subsection{Output feedback in the PD-NJ-ODE model}\label{sec:Output feedback in the PD-NJ-ODE model}
Using the output of a discrete dynamical system at time $t$ as additional input to the system at the following time $t+1$ is denoted as \emph{output feedback} in the literature of reservoir computing and known to stabilize the training of such dynamical systems \citep{reinhart2011reservoir}. In line with this, we propose to use output feedback in the PD-NJ-ODE framework and remark that this does not change the theoretical guarantees of the model. Indeed, the model can always just ignore this additional input, hence, the same results hold. However, the inductive bias when training the model with this additional input is better as we see in \Cref{sec:Experiments}.

\section{Experiments}\label{sec:Experiments}
The code with all experiments is available at \url{https://github.com/FlorianKrach/PD-NJODE}. For the experiments on synthetic stochastic datasets, we use the evaluation metric of \citet[Section~8]{krach2022optimal}.
On all synthetic datasets, we use a previously unseen and independent test set to evaluate the models.

In \Cref{sec:Long-term prediction of chaotic systems} we show that the enhanced training framework together with output feedback enables the PD-NJ-ODE to predict (deterministic) chaotic systems with great accuracy over a long time horizon. The enhanced training framework also leads to better long-term predictions for stochastic datasets, as shown on 3 examples in \Cref{sec:Long-term predictions in stochastic systems}.

\subsection{Long-term prediction of chaotic systems}\label{sec:Long-term prediction of chaotic systems}
We showcase the potential of our enhanced PD-NJ-ODE model for deterministic (differential) systems by applying it to the chaotic system of a double pendulum, that was already described and used in \citet[Appendix~B.3 \&~C.3]{krach2022optimal}.
This chaotic system can be described by an ODE in 4 variables (the two angles $\alpha_i$ of the pendulums and their two generalized momenta $p_i$). By choosing the initial value of $\alpha_1 = \alpha_2$ randomly around $\pi$ we introduce small deviations in the initial conditions of this chaotic system, which lead to highly diverse paths.
For more details on the setup of the experiment see \Cref{sec:Experimental Details}.

We use the same setting as in \citet[Appendix~C.3]{krach2022optimal} and compare the standard PD-NJ-ODE model (labelled ``N'') to 
\textbf{i)} the PD-NJ-ODE with output feedback (N-OF), \textbf{ii)} the PD-NJ-ODE with input skipping (N-IS), 
\textbf{iii)} the PD-NJ-ODE with output feedback and input skipping (N-OF-IS),  \textbf{iv)} the PD-NJ-ODE with increasing input skipping (N-IIS) and \textbf{v)} the PD-NJ-ODE with output feedback and increasing input skipping (N-OF-IIS).  In particular, N-IS refers to the model where none of the observations after $X_0$ are use as input and N-IIS refers to the procedure of \Cref{sec:Suggested training procedure}, where we define $p(E) = \max(0, 1-\tfrac{E}{100})$, where $E$ denotes the current training epoch. All of these models use the same architecture and are trained for $200$ epochs. Moreover, we additionally train the N-OF-IIS again with the same architecture, however with a 5 times larger dataset having a $2.5$ times larger observation probability and with $300$ epochs (N-OF-IIS-large).

\begin{table}[t]
\caption{Comparison of MSEs on test set of the different models.}
\label{tab:model comp}
\begin{center}
\begin{tabular}{l | r r r r r r | r}
\toprule
model   & N   & N-OF & N-IS   & N-OF-IS & N-IIS & N-OF-IIS   & N-OF-IIS-large \\ \midrule 
MSE     & 2.492 & 2.024  & 0.719 & 0.641    & 0.474  & 0.468 & 0.181 \\
\bottomrule
\end{tabular}
\end{center}
\end{table}

\begin{figure}[!tb]
\centering
\includegraphics[width=0.49\textwidth]{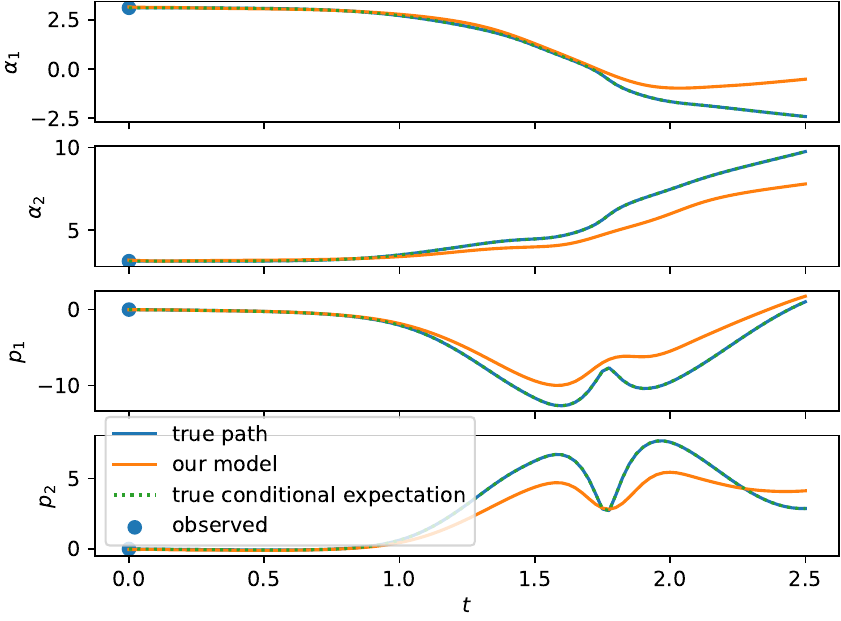}
\includegraphics[width=0.49\textwidth]{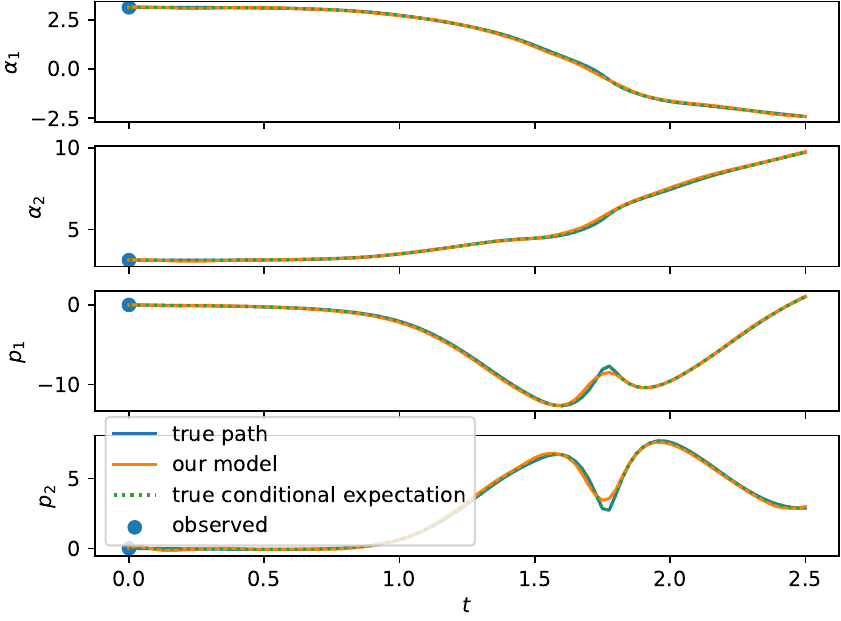}
\includegraphics[width=0.49\textwidth]{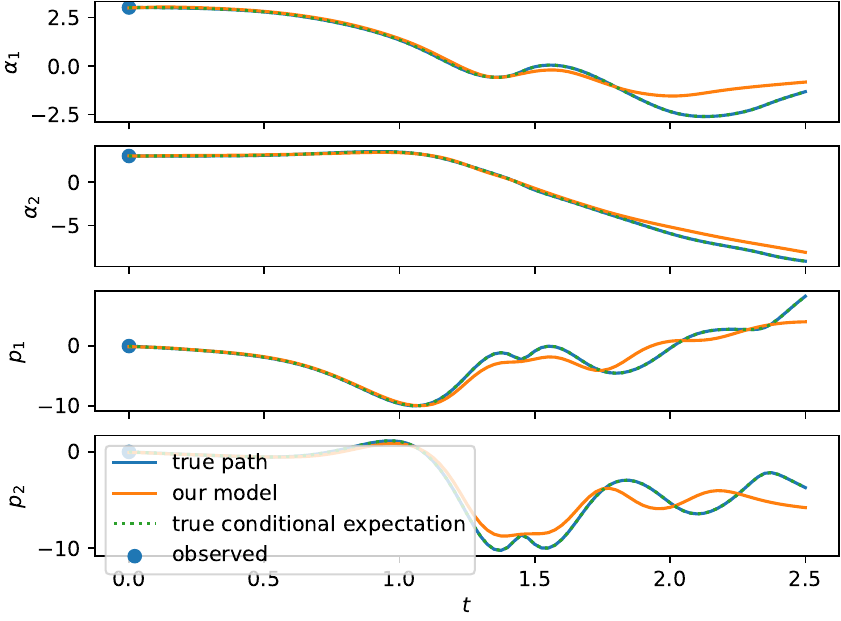}
\includegraphics[width=0.49\textwidth]{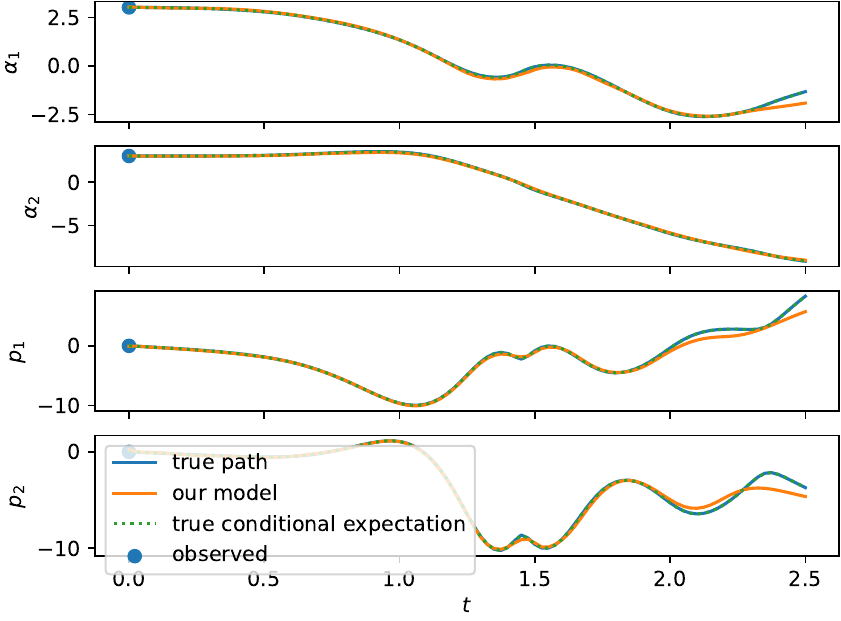}
\caption{Left: test samples of a Double Pendulum with standard training framework (N). Right: the same test samples of the Double Pendulum with the enhanced training framework and larger dataset (N-OF-IIS-large). The conditional expectation coincides with the process, since it is deterministic.}
\label{fig:DP Comp}
\end{figure}

We evaluate the trained models on the test set, by computing the MSE between their predictions and the true paths on a fine equidistant grid (the same as used for sampling the ODE paths).
The results are given in \Cref{tab:model comp}. In particular, we see that output feedback and input skipping independently of each other improve the results, where the impact of input skipping is larger than the one of output feedback.
Moreover, we see a clear increase in performance when switching from input skipping to increasing input skipping (with and without output feedback). In particular, this shows that the model benefits from the additional ``initial values'' used in the beginning of the training. Overall, the performance increases by more than a factor $5$ from N to N-OF-IIS and by more than a factor $13$ from N to N-OF-IIS-large.

In \Cref{fig:DP Comp} we show the comparison of N and N-OF-IIS-large on two samples of the test set. While the standard PD-NJ-ODE model starts to diverge from the true path after about half of the evaluation time, the enhanced PD-NJ-ODE model nearly perfectly predicts the path over the entire period.

\subsection{Long-term predictions in stochastic systems}\label{sec:Long-term predictions in stochastic systems}
We use 3 different geometric Brownian motion (Black--Scholes) dataset with similar specifics as in \citet{herrera2021neural}. 
Two of the datasets have constant drift and are identical except that they either use an observation probability of $10\%$ (BS-Base) or $40\%$ (BS-HighFreq). The $3$rd dataset uses a time-dependent drift and an observation probability of $10\%$ and is otherwise identical to the other datasets (BS-TimeDep).

\begin{figure}
    \centering
    \includegraphics[width=0.49\linewidth]{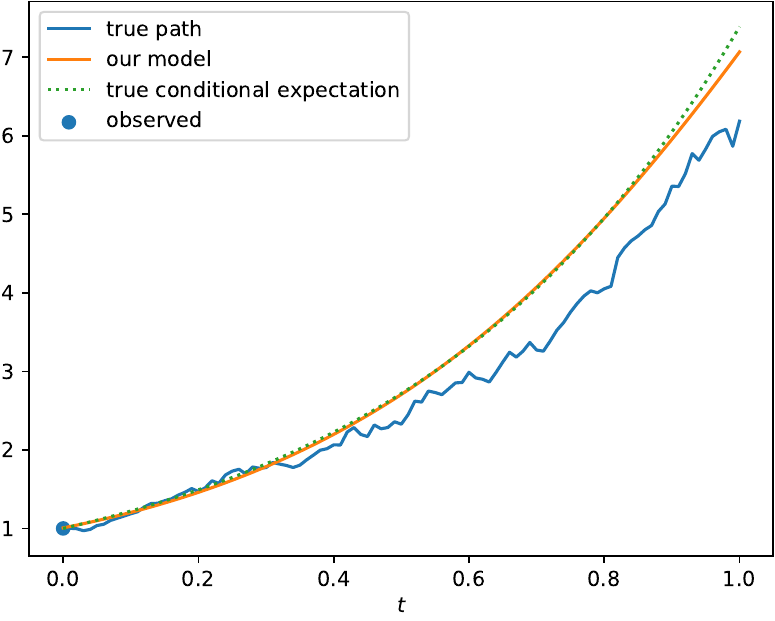}
    \includegraphics[width=0.49\linewidth]{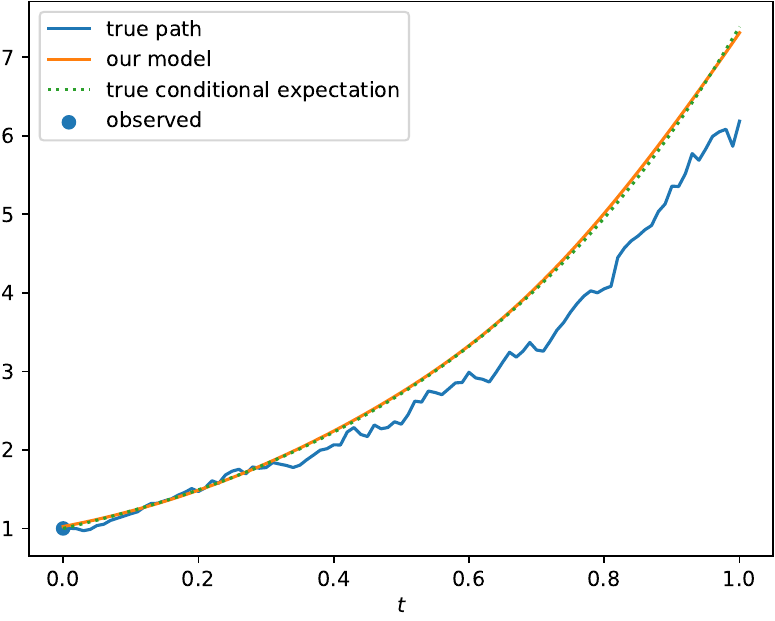}
    \includegraphics[width=0.49\linewidth]{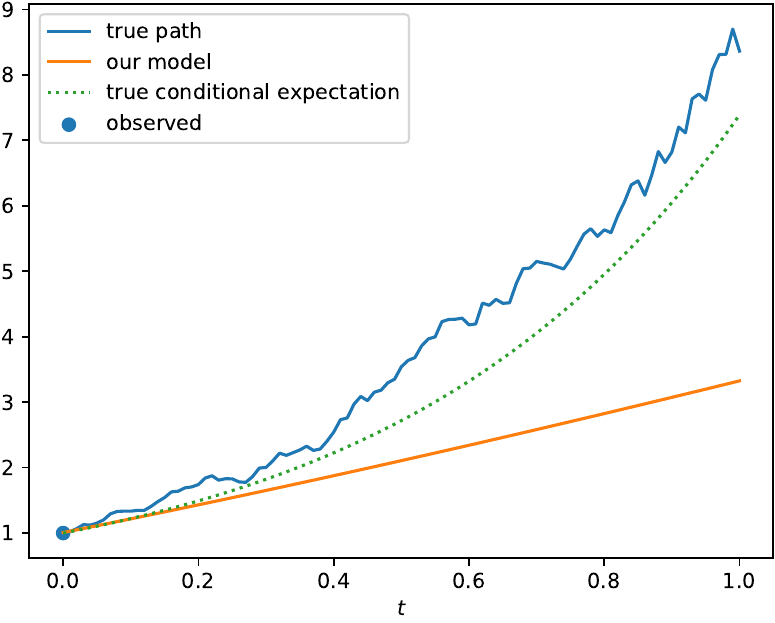}
    \includegraphics[width=0.49\linewidth]{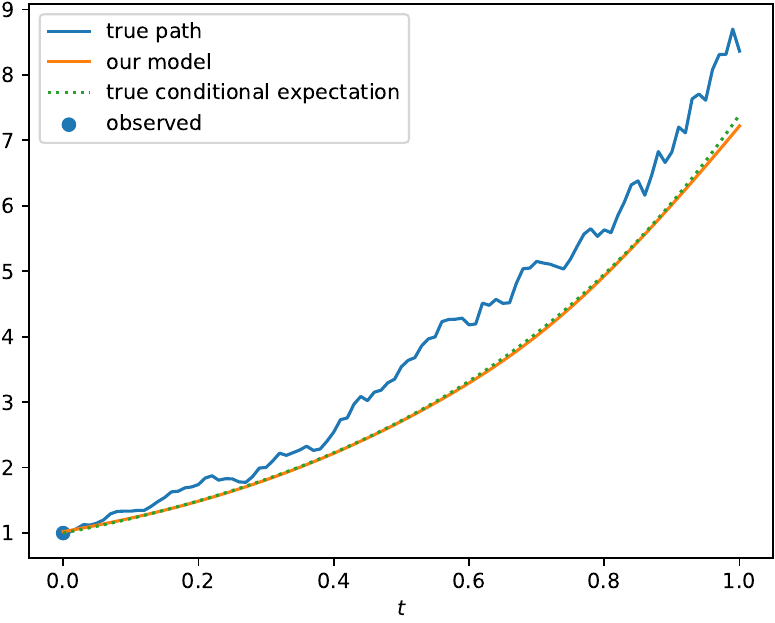}
    \includegraphics[width=0.49\linewidth]{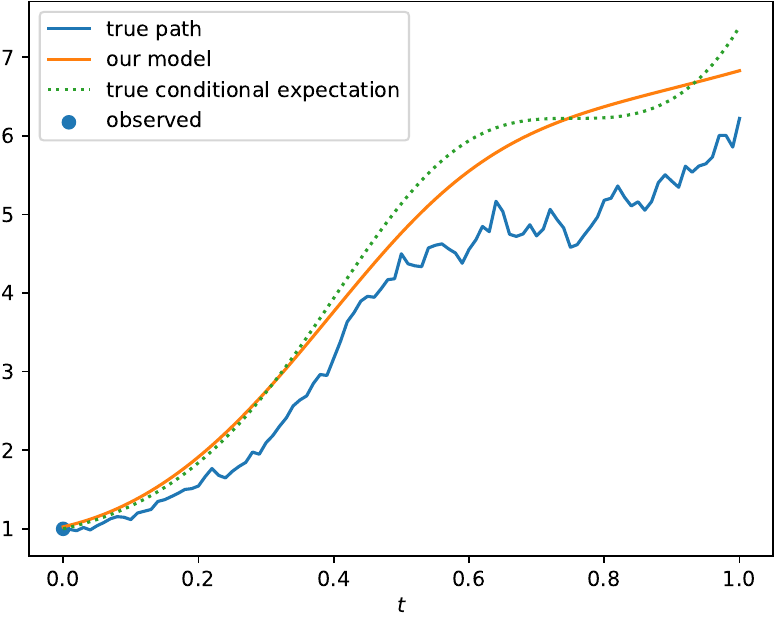}
    \includegraphics[width=0.49\linewidth]{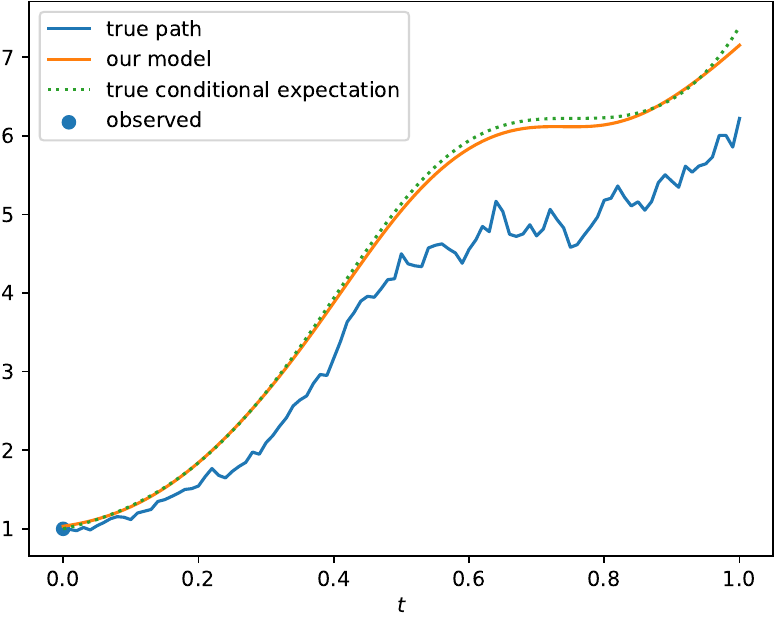}
    \caption{Comparison of the standard (N; left) and enhanced (N-OF-IIS; right) model on a test sample of the BS-Base (top), BS-HighFrequ (middle) and BS-TimeDep (bottom) dataset.}
    \label{fig:BS comparison}
\end{figure}

In the BS-Base dataset, each of the $100$ points of the sampling grid is randomly chosen as observation time with probability $10\%$. Hence, the probability of not having an observation for $100$ consecutive steps is smaller than $0.01\%$. Therefore, it is very unlikely that the model will learn to correctly predict for such a long time (without intermediate observations), when trained with the standard training framework. 
For the BS-HighFreq dataset, this probability is further reduced to below $10^{-22}$, making it even more unlikely that the standard model will learn to correctly predict over long terms.
The difficulty of the BS-TimeDep dataset is that the dynamic changes with time (as in the chaotic system dataset). This makes it more difficult for the standard model to learn, when observations are not far enough apart.
The enhanced training framework (\Cref{sec:Long-term predictions}) should allow the model to circumvent these challenges, as shown theoretically. 
We compare the standard PD-NJ-ODE (N) with the PD-NJ-ODE with output feedback and increasing input skipping (N-OF-IIS), where an observation is used as input to the model with probability $p(E) = \max(0, 1-\tfrac{E}{100})$, decreasing with the training epoch $E$ during the $200$ epochs of training. 

We evaluate and compare both models on the test sets of the 3 datasets and see in \Cref{tab:model comp 2} that the enhanced training framework leads to large improvements in terms of the evaluation metric.
\begin{table}[t]
\caption{Comparison of the minimal evaluation metric of the standard (N) and the enhanced (N-OF-IIS) PD-NJ-ODE model on different datasets.}
\label{tab:model comp 2}
\begin{center}
\begin{tabular}{l | r r r }
\toprule
        & BS-Base ($\times 10^{-3}$)  & BS-HighFreq ($\times 10^{-3}$) & BS-TimeDep ($\times 10^{-2}$) \\ \midrule 
N           &  3.59 & 2506.29 & 3.52  \\
N-OF-IIS    &  0.58 & 0.37 & 0.23 \\
\bottomrule
\end{tabular}
\end{center}
\end{table}
Moreover, this improvement is also well visible in \Cref{fig:BS comparison}. For BS-Base we see the (slightly) degrading performance of the standard model N approaching the time horizon, which is not prevalent for the enhanced model N-OF-IIS.
On the BS-HighFreq dataset the standard model N performs much worse, diverging from the true conditional expectation already after a short time, while the enhanced model predicts nearly perfectly. This was expected, since the model N is much less exposed to predicting over longer time intervals during the training. 
Finally, in contrast to N-OF-IIS, the standard model does not learn the correct dynamic in the long run on the BS-TimeDep dataset.
Comparing the results of N-OF-IIS on the BS-Base and BS-HighFreq dataset, it might seems  surprising at first that the model performs better on the latter dataset, where the model N performs much worse. However, this can be explained by the much larger number of observed samples available in BS-HighFreq that the model can make use of with the enhanced training procedure.

\section{Conclusion}
While it has been known before that the PD-NJ-ODE model can be used to learn (chaotic) deterministic systems, given for example through ODEs or PDEs, a limiting factor for the use in practice was the degrading prediction accuracy for increasing prediction time. In this work we proposed two enhancements of the PD-NJ-ODE model as a remedy for this problem. Simultaneously, these enhancements also enable long-term predictions with the PD-NJ-ODE model in the case of generic stochastic datasets.
In particular, convergence of the model output to a much more general conditional expectation process (with arbitrary sub-information) is guaranteed by the suggested new training procedure. Since there are no known drawbacks, the use of this new training procedure is always recommended.

\section*{Acknowledgement}
The authors thank Jakob Heiss for many deep and insightful discussions about the topics treated in this work and related topics.

\bibliography{references}
\bibliographystyle{iclr2024_workshop}

\appendix
\section{Experimental details}\label{sec:Experimental Details}

\subsection{Double Pendulum}
\paragraph{Dataset.}
We explain the chaotic system of a double pendulum, depicted in Figure~\ref{fig:double pendulum}, following \citet{DoublePendulum, krach2022optimal}.
\begin{figure}[htp!]
\centering
\includegraphics[width=0.49\textwidth]{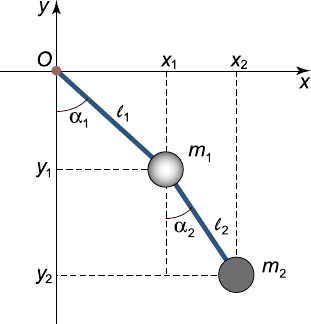}
\caption{A schematic representation of a double pendulum. Picture copied from \citet{DoublePendulum}.}
\label{fig:double pendulum}
\end{figure}
The dynamical system is determined completely by a $4$-dimensional state vector $(\alpha_1, \alpha_2, p_1, p_2)$, where $(\alpha_1, \alpha_2)$ determine the current position of both pendulums and $(p_1, p_2)$ are the so-called generalized momenta, which are related to the velocities of both pendulums. This state vector satisfies the differential system
\begin{align*}
\alpha_1^\prime & = \frac{p_1 l_2 - p_2 l_1 \cos(\alpha_1 - \alpha_2)}{l_1^2 l_2 A_0}, \\
\alpha_2^\prime & = \frac{p_2(m_1+m_2) l_1 - p_1 m_2 l_2 \cos(\alpha_1 - \alpha_2)}{m_2 l_1 l_2^2 A_0 } ,\\
p_1^\prime &= - (m_1 + m_2) g l_1 \sin(\alpha_1) - A_1 + A_2, \\
p_2^\prime &= - m_2 g l_2 \sin(\alpha_2) +A_1 - A_2,
\end{align*}
where
\begin{align*}
A_0 &= [m_1 + m_2 \sin^2(\alpha_1 - \alpha_2)], \\
A_1 &= \frac{p_1 p_2 \sin(\alpha_1 - \alpha_2)}{l_1 l_2 A_0}, \\
A_2 &= \frac{[p_1^2 m_2 l_2^2 - 2 p_1 p_2 m_2 l_1 l_2 \cos(\alpha_1 - \alpha_2) + p_2^2 (m_1 + m_2) l_1^2 ] \sin(2(\alpha_1 - \alpha_2))}{2 l_1^2 l_2^2 A_0^2},
\end{align*}
and $g$ is the gravitational acceleration constant.

We choose $m_1=m_2=l_1=l_2 =1$ and only consider positions where the double pendulum is straight, i.e., both pendulums have the same angle $\alpha := \alpha_1 = \alpha_2$, and the generalized momenta  $p_1, p_2$ are $0$, as initial conditions $X_0$. Hence, the initial points are sampled by randomly sampling $\alpha$ from some distribution on $[0, 2\pi]$. We use $\alpha \sim N(\pi, 0.2^2)$, i.e., normally distributed around to highest point the Pendulum can reach.
For each initial point, we sample a path from the ODE using the Runge--Kutta method of order 4 (RK4) with step size $0.025$ on the time interval $[0,2.5]$, which leads to $101$ time points. Each time point is independently chosen as observation with probability $0.1$. Overall, we sample $20K$ paths out of which $20\%$ are used as validation set. For N-OF-ISS-large, a dataset with $100K$ samples and observation probability $0.25$ is used.

\paragraph{Architecture.}
We use the PD-NJ-ODE, with the following architecture. 
The latent dimension is $d_H = 400$ and all 3 neural networks (encoder, neural ODE and readout network) have the same structure of $1$ hidden layers with $\tanh$ activation function and $200$ nodes. Empirically, the model performed best when using  the recurrent jump network, but no signature terms as input. 

\paragraph{Training.}
All models are trained for $200$ epochs, except for N-OF-ISS-large, which is trained for $300$ epochs. Early stopping is performed based on the loss on the validation set.

\subsection{Geometric Brownian motion datasets}
The datasets are the same as the Black--Scholes datasets in \citet[Section~6.1 and~6.2]{herrera2021neural}.

\paragraph{Dataset.}
The geometric Brownian motion is defined by the SDE
\begin{equation*}
    dX_t = \mu X_t dt + \sigma X_t dW_t,
\end{equation*}
where $W$ is a Brownian motion. For all datasets, we use $\sigma = 0.3$, $X_0 = 1$, and sample $20'000$ paths using the Euler-Maruyama method with time step $\Delta t = 0.01$ on the time interval $[0,T]$. For BS-Base and BS-HighFreq we choose drift $\mu = 2$, while we use the time dependent drift $\mu(t) = \sin(2 \pi t) + 1$ for the BS-TimeDep dataset.
Each time point is independently chosen as observation with probability $0.1$ for the datasets BS-Base and BS-TimeDep, and with probability $0.4$ for BS-HighFreq (leading to shorter intervals between any two observations).

\paragraph{Architecture.}
We use the PD-NJ-ODE, with the following architecture. 
The latent dimension is $d_H = 100$ and all 3 neural networks (encoder, neural ODE and readout network) have the same structure of $1$ hidden layers with $\tanh$ activation function and $100$ nodes. The model uses a recurrent jump network and the signature terms of level $3$ as input. 

\paragraph{Training.}
All models are trained for $200$ epochs. Early stopping is performed based on the loss on the validation set.

\end{document}